%% file: TRNN.tex
\documentclass{article}
\usepackage{spconf}
\usepackage{url}
\usepackage{graphicx}
\usepackage{caption}
\usepackage{array}
\usepackage{subfig}
\usepackage{cite}
\usepackage{stmaryrd}
\usepackage{amsfonts,latexsym,amssymb}
\usepackage[cmex10]{amsmath}
\usepackage{algorithm}
\usepackage{algorithmic}
\usepackage{hyperref}
\usepackage{multirow}
\usepackage{arydshln}

\usepackage[OT1]{fontenc}
\usepackage{amsthm}
\newtheorem{definition}{{Definition}}

\newtheorem{theorem}{{Theorem}}

\begin{document}
\input{Abstract}
\input{Introduction}
\input{Notations}

\input{Model}
\input{Experiments}
\input{Conclusion}

\newpage

\bibliographystyle{IEEEbib}
\bibliography{TRNN}

\end{document}

%% file: Abstract.tex
%
\title{Tensor-Ring Nuclear Norm Minimization and Application for Visual Data Completion}

\name{Jinshi Yu$^{1,2}$, Chao Li$^2$, Qibin Zhao$^{1,2\star}$, Guoxu Zhou$^{1\star}$\thanks{$\star$ Corresponding authors. E-mail address: gx.zhou@gdut.edu.cn, qibin.zhao@riken.jp}}
\address{$^1$School of Automation, Guangdong University of Technology, Guangzhou, 510006, China\\
$^2$RIKEN Center for Advanced Intelligence Project (AIP), Tokyo, 103-0027, Japan}

%
%


\maketitle


\begin{abstract}
Tensor ring (TR) decomposition has been successfully used to obtain the state-of-the-art performance in the visual data completion problem.
However, the existing TR-based completion methods are severely non-convex and computationally demanding. In addition, the determination of the optimal TR rank is a tough work in practice.
To overcome these drawbacks, we first introduce a class of new tensor nuclear norms by using \emph{tensor circular unfolding}.
Then we theoretically establish connection between the rank of the circularly-unfolded matrices and the TR ranks. We also develop an efficient tensor completion algorithm by minimizing the proposed tensor nuclear norm.
Extensive experimental results demonstrate that our proposed tensor completion method outperforms the conventional tensor completion methods in the image/video in-painting problem with striped missing values.
 
\end{abstract}

\begin{keywords}
Tensor completion, tensor ring decomposition, nuclear norm, image in-painting.
\end{keywords}

%% file: Introduction.tex
\section{Introduction}
Low rank tensor completion (LRTC) problem aims to recover the incomplete tensor from the observed entries by assuming different low-rank tensor structures, and it has attracted a lot of attentions in the past decades \cite{liu2013tensor, liu2009tensor, wang2017efficient, yuan2018higher, yuan2018tensor}. 
Most recently, Zhao \textit{et al.} proposed tensor ring decomposition~\cite{zhao2016tensorring}, which achieves the state-of-the-art performance in the LRTC problem~\cite{wang2017efficient,yuan2018tensor,yuan2018higher}.

However, the drawbacks limit the application of the existing TR-based methods in practice.
One of the drawbacks is that the existing TR-based methods are quite time-consuming. For example, TR-ALS \cite{wang2017efficient} has the computational complexity of $\mathcal{O}(PNR^{4}I^{N}+NR^{6})$ and the one of tensor ring low-rank factors (TRLRF)\cite{yuan2018tensor} equals $\mathcal{O}(NR^{2}I^{N}+NR^{6})$, where $N$ denotes the order of the tensor, $I$ denotes the dimension of each mode, $R$ represents the rank of the model and $P$ is a constant between 0 and 1. It can be seen that the computational cost of the two methods increase with the sixth power of TR rank $R$. It implies that a large TR-rank chosen in practice leads to terribly low efficiency, and sometimes with the memory-exploration problem. Though tensor ring weighted optimization (TR-WOPT) \cite{yuan2018higher} applies gradient descent algorithm to find the latent core tensors, the convergence rate of this method is low. Beside the computational complexity problem, the determination of the optimal TR rank is also a tough work in the completion problem. It is because that TR rank is defined as a vector, whose dimension equals to the order of the tensor. This fact makes that the computational complexity for rank selection exponentially increases with the dimension of the rank by using cross validation. Furthermore, TR decomposition is non-convex, so there is no theoretical guarantee to obtain the global minimum solution. 

\begin{figure}
\begin{center}
\includegraphics[width=0.25\textwidth]{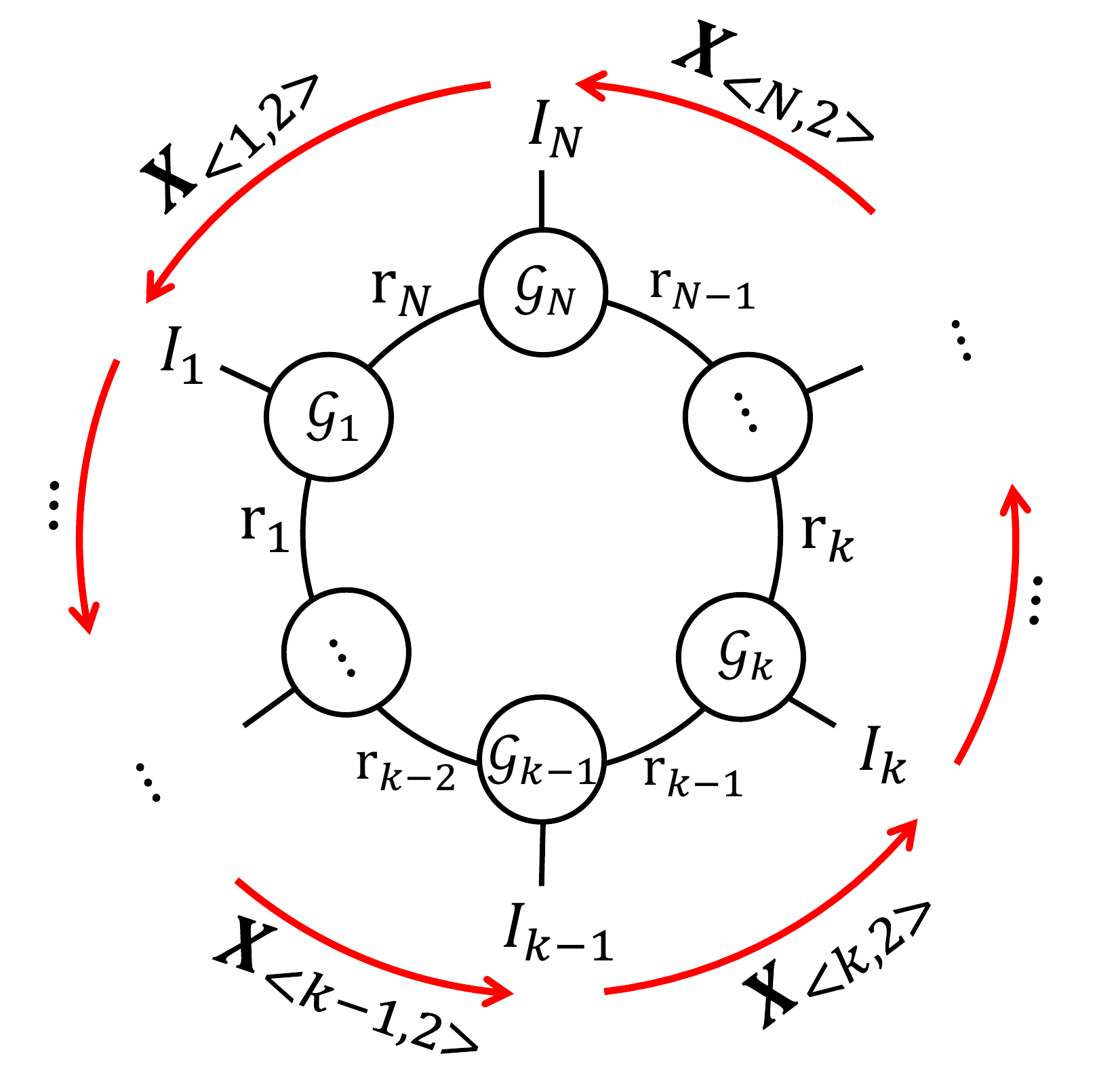}\vspace{-0.2cm}
\caption{Illustration of TR representation of a $N$th-order tensor $\mathcal{X}\in\mathbb{R}^{I_{1}\times\dots\times I_{N}}$ and its {\it tensor circular unfolding}'s. The nodes $\mathcal{G}_{k}\in\mathbb{R}^{r_{k-1}\times I_{k}\times r_{k}}$ for ${k=1\cdots N}$ \cite{zhao2016tensorring} represent tensors whose order is denoted by the number of edges. Each dimension of the tensor is specified by the number beside the edges. The connection line between two nodes denotes the multi-linear product operation of these two nodes along a specific mode. ${\bf X}_{<k, 2>}$ represents the {\it tensor circular unfolding} along modes $k-1$ and $k$, specified by a red arc. Note that {\it tensor circular unfolding}'s are obtained by circularly shifting along the tensor ring. }\label{archi:kdUnfolding}\vspace{-0.8cm}
\end{center}
\end{figure}

To overcome these drawbacks, we develop a novel convex completion method by minimizing tensor ring nuclear norm defined in this paper.
Specifically, we first define a new {\it circular unfolding} operation on higher-order tensor, and theoretically prove that ranks of the circularly-unfolded matrices bound their corresponding TR rank.
After that, the tensor ring nuclear norm is defined as a sum of the matrix nuclear norm of the circularly-unfolded matrices.
As a convex surrogate of TR rank, the proposed completion method not only has lower computational complexity than the conventional TR-based methods, but also avoids choosing the optimal TR rank manually.
To sum up, our contributions of this paper are listed below:
\begin{itemize}\vspace{-0.16cm}
\item We theoretically prove the relationship between the TR-rank and the rank of the circularly-unfolded matrix.\vspace{-0.16cm}
\item To our best knowledge, this is the first paper to introduce tensor ring nuclear norm, and it is demonstrated to obtain the state-of-the-art performance in the image and video completion problem.\vspace{-0.16cm}
\item An alternating direction method of multipliers (ADMM) based algorithm is developed to optimize the proposed model.
\end{itemize}

 

%% file: Notations.tex
\section{Tensor Ring Nuclear Norm}\label{section2}
To introduce the TR nuclear norm formulation, we first define the {\it tensor circular unfolding} and then theoretically reveal its connection to TR rank.
\begin{definition}\label{def:kd_unfolding}
(Tensor circular unfolding) Let $\mathcal{X}\in\mathbb{R}^{I_{1}\times\dots\times I_{N}}$ be a $N$th-order tensor, its tensor circular unfolding is a matrix, denoted by ${\bf X}_{<k, d>}$ of size $I_{t}I_{t+1}\dots I_{k} \times I_{k+1}\dots I_{t-1},$ whose elements are defined by
\vspace{-0.2cm}\begin{eqnarray}
{\bf X}_{<k, d>}(i_{t}i_{t+1}\dots i_{k},i_{k+1}\dots i_{t-1})=\mathcal{X}(i_{1},i_{2},\dots,i_{N})
\end{eqnarray}
where 
\vspace{-0.2cm}\begin{equation}\label{eq:t}
t = \left\{ \begin{array}{ll}
k-d+1, \ \ \quad\qquad d\le k ;\\
k-d+1+N \qquad  \text{otherwise} .
\end{array} \right.
\end{equation}
$d$ continuous indices (including $k$th index) enumerate the rows of $X_{<k, d>}$, and the rest $N-d$ indices for its columns. $d<N$ is the positive integer and named step-length in our paper.
\end{definition}

\begin{theorem}\label{theo:boundTR}
Assume $\mathcal{X}\in\mathbb{R}^{I_{1}\times\dots I_{N}}$ is $N$th-order tensor with $\left[r_{1}, r_{2}, \dots, r_{N}\right]$ TR-format, then for each unfolding matrix ${\bf X}_{<k, d>}$, 
\vspace{-0.2cm}\begin{eqnarray}
{\it rank}\left({\bf X}_{<k, d>}\right)\le r_{k}r_{t-1},
\end{eqnarray}
where $r_{0}=r_{N}$.
\end{theorem}
\begin{proof}
The tensor $\mathcal{X}$ with TR-format is expressed in element-wise form given by 
\begin{eqnarray}\label{eq:trX}
\mathcal{X}(i_{1}, i_{2}, \dots, i_{N}) = \text{Tr}\left\{ {\bf G}_{1}(i_{1}) {\bf G}_{2}(i_{2})\dots  {\bf G}_{N}(i_{N}) \right\}.
\end{eqnarray}
${\bf G}_{k}(i_{k})$ denotes the $i_{k}$th lateral slice matrix of the latent tensor $\mathcal{G}_{k}$, which is of size $r_{k-1}\times r_{k}$. By employing the property of the trace operation, the element of tensor $\mathcal{X}$ can be expressed in {\it tensor circular unfolding} format, i.e., 
\vspace{-0.2cm}\begin{eqnarray}\label{eq:trX2}
&{\bf X}_{<k, d>}(i_{t} i_{t+1} \dots, i_{k}, i_{k+1} \dots i_{t-1} )\nonumber\\
=& \text{Tr}\{ {\bf G}_{t}(i_{t})\dots  {\bf G}_{k}(i_{k}) {\bf G}_{k+1}(i_{k+1})\dots {\bf G}_{t-1}(i_{t-1})\}\nonumber\\
=& \text{Tr}\{{\bf A}(i_{t}i_{t+1}\dots i_{k}) {\bf B}(i_{k+1}\dots i_{t-1})\},
\end{eqnarray}
where $\mathcal{A}\in\mathbb{R}^{r_{t-1}\times I_{t}I_{t+1}\dots I_{k}\times r_{k}}$ with ${\bf A}(I_{t}I_{t+1}\dots I_{k}) = {\bf G}_{t}(i_{t}) {\bf G}_{t+1}(i_{t+1})\dots  {\bf G}_{k}(i_{k})$, $\mathcal{B}\in\mathbb{R}^{r_{k}\times I_{k+1}\dots I_{t-1}\times r_{t-1}}$ with ${\bf B}(I_{k+1}\dots I_{t-1}) =  {\bf G}_{k+1}(i_{k+1})\dots {\bf G}_{t-1}(i_{t-1}).$
\\ We can also rewrite (\ref{eq:trX2}) in the index form, which is 
\begin{eqnarray}\label{eq:trX3}
&{\bf X}_{<k, d>}(i_{t}i_{t+1}\dots, i_{k}, i_{k+1} \dots i_{t-1} )\nonumber\\
=& \sum_{\alpha_{k} = 1}^{r_{k}}\sum_{\alpha_{t-1} = 1}^{r_{t-1}}\mathcal{ A}(\alpha_{t-1}, i_{t}i_{t+1}\dots i_{k}, \alpha_{k}) \nonumber\\ 
&\mathcal{B}(\alpha_{k}, i_{k+1}\dots i_{t-1}, \alpha_{t-1})
\nonumber\\
=& \sum_{\alpha_{k}\alpha_{t-1} = 1}^{r_{k}r_{t-1}} {\bf \hat A}( i_{t}i_{t+1}\dots i_{k}, \alpha_{k}\alpha_{t-1})
\nonumber\\
& {\bf \hat B}( i_{k+1}\dots i_{t-1}, \alpha_{k}\alpha_{t-1})
\end{eqnarray}
Therefore, we can get ${\it rank}\left({\bf X}_{<k, d>}\right)\le r_{k}r_{t-1}.$
\end{proof}

\begin{definition}\label{def:TRNN}
(TR nuclear norm) Assume the tensor $\mathcal{X}$ with TR-form, its TR nuclear norm is defined by
\begin{eqnarray}
Q_{tr} = \sum_{k=1}^{N}\alpha_{k} \|{\bf X}_{<k, d>}\|_{*},
\end{eqnarray}
where $\|.\|_{*}$ is defined as the sum of singular values of a matrix.
\end{definition}

Note that TR nuclear norm is combined by a series of {\it tensor circular folding}'s $\left\{ {\bf X}_{<k, d>} \right\}_{k=1}^{N}$ of a tensor. In the case of  $d=2$, $\left\{{\bf X}_{<k, 2>}\right\}_{k=1}^{N}$ is obtained by circularly shifting along the tensor ring, shown in Fig. \ref{archi:kdUnfolding}.

%% file: Model.tex
\section{Tensor Ring Nuclear Norm Minimization}\label{section3}
By utilizing the relationship between the TR-rank and the rank of the circularly-unfolded matrices, a novel convex model named TR nuclear norm minimization (TRNNM) is proposed for LRTC problem, i.e.,
\vspace{-0.2cm}\begin{eqnarray}\label{model:TRNN}
\min_{\mathcal{X}}:&\sum_{k=1}^{N}\alpha_{k} \|{\bf X}_{<k, d>}\|_{*}\nonumber\\
s.t.:&\mathcal{X}_{\Omega} = \mathcal{T}_{\Omega}
\end{eqnarray}
where $\Omega$ denotes the index set of observed entries of $\mathcal{T}.$

The problem (\ref{model:TRNN}) is difficult to solve due to $\{{\bf X}_{<k, d>}\}_{k=1}^{N}$ share the same entries and  can't be optimized independently. To simplify the optimization, we introduce additional tensors $\{\mathcal{M}^{(k)}\}_{k=1}^{N}$ and thus obtain the equivalent formulation:
\vspace{-0.2cm}\begin{eqnarray}\label{model:TRNN-1}
\min_{\mathcal{X}, \mathcal{M}^{(k)}}:&\sum_{k=1}^{N}\alpha_{k} \|{\bf M}_{<k, d>}^{(k)}\|_{*}\nonumber\\
s.t.:&\mathcal{X}=\mathcal{M}^{(k)}, k=1,\dots,N\nonumber\\
& \mathcal{X}_{\Omega} = \mathcal{T}_{\Omega}
\end{eqnarray}

\renewcommand{\algorithmicensure}{\textbf{Parameters:}}
\begin{algorithm}[htb]
\caption{The TRNNM algorithm.}
\label{alg:TRNNM}
\begin{algorithmic}[1]
\REQUIRE
       { Missing entry zero filled tensor data ${\mathcal T}\in\mathbb{R}^{I_{1}\times\dots\times I_{N}}$ with observed index set $\Omega$, 
\ENSURE
       $\{\alpha_{k}\}_{k=1}^{N}$, $\rho=1e-5$, $tol=1e-5$ .}

 \STATE{ Initialize: zero filled $\mathcal{X}$, $\mathcal{M}^{(k)}$ with $\mathcal{M}^{(k)}_{\Omega} = \mathcal{X}_{\Omega} =\mathcal{T}_{\Omega}$}, $\mathcal{Y}^{(k)}=0$.

 \FOR{$t=1 $ \TO$t_{max}$}
 \STATE{$\hat{\mathcal{X}}=\mathcal{X}$}
 \FOR{$k=1$ \TO $N$}
 \STATE{$\mathcal{M}^{(k)}\gets(\ref{eq:m})$}
 \ENDFOR
 \STATE{$\mathcal{X}\gets(\ref{eq:x})$}
 \FOR{$k=1$ \TO $N$}
 \STATE{$\mathcal{Y}^{(k)}\gets(\ref{eq:y})$}
 \ENDFOR
 \IF{$\|\mathcal{X}-\hat{\mathcal{X}}\|_{F}/\|\hat{\mathcal{X}}\|_{F}\le tol$}
 \STATE{break}
 \ENDIF
 \ENDFOR
\end{algorithmic}
\end{algorithm}

ADMM is developed to solve problem (\ref{model:TRNN-1}) due to its efficient in solving optimization problem with multiple nonsmooth terms in the objective function \cite{lin2010augmented}. We define the  augmented Lagrangian function as follows:
\begin{equation}\label{model:TRNN-lagr}
\begin{split}
&L(\mathcal{X},\mathcal{M}^{(1)},\dots, \mathcal{M}^{(N)}, \mathcal{Y}^{(1)},\dots,\mathcal{Y}^{(N)})\\
&=\sum_{k}^{N}\alpha_{k}\|{\bf M}_{<k, d>}^{(k)}\|_{*} + <\mathcal{X}-\mathcal{M}^{(k)},\mathcal{Y}^{(k)}>\\
&+\frac{\rho}{2}\|\mathcal{M}^{(k)}-\mathcal{X}\|_{F}^{2}\\
& s.t. \quad\mathcal{X}_{\Omega} = \mathcal{T}_{\Omega}
\end{split}
\end{equation}

According to the framework of ADMM, we can update $\mathcal{M}^{(k)}$'s, $\mathcal{X}$ and $\mathcal{Y}^{(k)}$'s as follows.\\
\\ {\bf{Update}} $\mathcal{M}^{(k)}$. It is easy to note that problem (\ref{model:TRNN-lagr}) can be converted to an equivalent formulation:
\begin{equation}\label{model:TRNN-lagr-mx}
\begin{split}
&L(\mathcal{X},\mathcal{M}^{(1)},\dots, \mathcal{M}^{(N)}, \mathcal{Y}^{(1)},\dots,\mathcal{Y}^{(N)})\\
&=\sum_{k}^{N}\alpha_{k}\|{\bf M}_{<k, d>}^{(k)}\|_{*} + \frac{\rho}{2}\|\mathcal{M}^{(k)}-\mathcal{X}-\frac{1}{\rho}\mathcal{Y}^{(k)}\|_{F}^{2} \\
&- \frac{1}{2\rho}\|\mathcal{Y}^{(k)}\|_{F}^{2}\\
\end{split}
\end{equation}
To optimize $\mathcal{M}^{(k)}$ is equivalent to solve the subproblem:
\begin{eqnarray}\label{model:TRNN-m}
\min_{\mathcal{M}^{(k)}}:  \sum_{k}^{N}\alpha_{k}\|{\bf M}_{<k, d>}^{(k)}\|_{*} + \frac{\rho}{2}\|\mathcal{M}^{(k)}-\mathcal{X}-\frac{1}{\rho}\mathcal{Y}^{(k)}\|_{F}^{2}
\end{eqnarray}

The above problem has been proven to lead to a closed form in \cite{ma2011fixed, cai2010singular, recht2010guaranteed}. Thus the optimal $\mathcal{M}^{(k)}$ can be given by:
\begin{eqnarray}\label{eq:m}
\mathcal{M}^{(k)}=\text{fold}_{k}\left [{\bf D}_{\tau}\left({\bf X}_{<k, d>} + \frac{1}{\rho}{\bf Y}_{<k, d>}\right)\right ]
\end{eqnarray}
where  $\tau = \frac{\alpha_{k}}{\rho}$ and ${\bf D}_{\tau}\left(\cdot\right)$ denotes the thresholding SVD  operation \cite{cai2010singular}. If the SVD of ${\bf A}={\bf USV}^{T}$, 
\begin{eqnarray}\label{eq:tsvd}
{\bf D}_{\tau}\left({\bf A}\right) = {\bf U}\max\left\{{\bf S}-\tau {\bf I},0\right\}{\bf V}^{T}
\end{eqnarray}
where ${\bf I}$ is an identity matrix with the same size of ${\bf S}.$
\\{\bf{Update}} $\mathcal{X}$. The optimal $\mathcal{X}$ with all other variables fixed is given by solving the following subproblem of (\ref{model:TRNN-lagr-mx}):
\begin{eqnarray}\label{model:TRNN-lagr-x}
\min_{\mathcal{X}}: & \frac{\rho}{2}\|\mathcal{M}^{(k)}-\mathcal{X}-\frac{1}{\rho}\mathcal{Y}^{(k)}\|_{F}^{2}\nonumber\\
 s.t.:& \quad\mathcal{X}_{\Omega} = \mathcal{T}_{\Omega}
\end{eqnarray}
It is easy to check that the solution of (\ref{model:TRNN-lagr-x}) is given by:
\begin{equation}\label{eq:x}
\mathcal{X}_{i_{1},\dots,i_{N}} = \left\{ \begin{array}{ll}
\left(\frac{1}{N}\sum_{k=1}^{N}\mathcal{Z}^{(k)}\right)_{i_{1},\dots,i_{N}} (i_{1},\dots,i_{N})\notin\Omega ;\\
\mathcal{M}_{i_{1},\dots,i_{N}}  \quad \quad \qquad\qquad (i_{1},\dots,i_{N})\in\Omega .
\end{array} \right.
\end{equation}
where $\mathcal{Z}^{(k)}=\mathcal{M}^{(k)}-\frac{1}{\rho}\mathcal{Y}^{(k)}.$
\\{\bf Update} $\mathcal{Y}^{(k)}$. The Lagrangian multiplier $\mathcal{Y}^{(k)}$ is updated by:
\begin{eqnarray}\label{eq:y}
\mathcal{Y}^{(k)} = \mathcal{Y}^{(k)} + \rho(\mathcal{X}-\mathcal{M}^{(k)})
\end{eqnarray}

The TRNNM algorithm is summarized in Algorithm \ref{alg:TRNNM}.
\subsection{Computational complexity of algorithm}
For a tensor $\mathcal{X}\in\mathbb{R}^{I_{1}\times\dots\times I_{N}}$ with $I_{k}=I, k=1,\dots, N,$ the computational complexity of our proposed method is $\mathcal{O}(NI^{N+d})$ where $d\le N/2$. In contrast to TRALS and TRLRF with computational complexity of $\mathcal{O}(PNR^{4}I^{N}+NR^{6})$ and $\mathcal{O}(NR^{2}I^{N}+NR^{6})$ respectively, the computational complexity of our method is may much smaller when $I\le R$. In practice, as shown in \cite{wang2017efficient, yuan2018tensor}, the suitable TR-rank $R$ is always much higher than the dimension $I$ in the high-order form of visual data, which is also found in our experiment. In addition, due to the immense TR-rank selection possibilities, the computational complexity of TR-ALS and TRLRF exponentially increases by using cross validation.

%% file: Experiments.tex
\vspace{-0.2cm}\begin{figure}[htb]
\centering
\begin{minipage}[t]{0.1\textwidth}
\includegraphics[width=1.0\linewidth, height=1.0\linewidth]{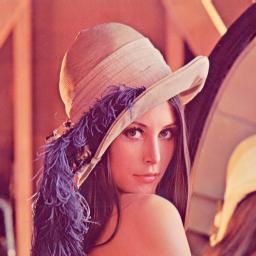}\vspace{-0.3cm}\caption*{Original}
\end{minipage}
\begin{minipage}[t]{0.1\textwidth}
\includegraphics[width=1.0\linewidth, height=1.0\linewidth]{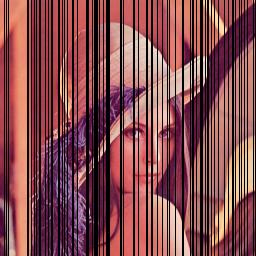}\vspace{-0.3cm}\caption*{$mr=0.3$}
\end{minipage}
\begin{minipage}[t]{0.1\textwidth}
\includegraphics[width=1.0\linewidth, height=1.0\linewidth]{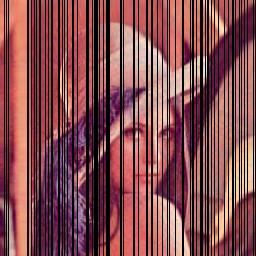}\vspace{-0.3cm}\caption*{FBCP}
\end{minipage}
\begin{minipage}[t]{0.1\textwidth}
\includegraphics[width=1.0\linewidth, height=1.0\linewidth]{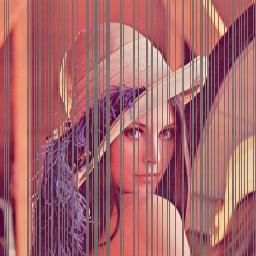}\vspace{-0.3cm}\caption*{HaLRTC}
\end{minipage}\\
\begin{minipage}[t]{0.1\textwidth}
\includegraphics[width=1.0\linewidth, height=1.0\linewidth]{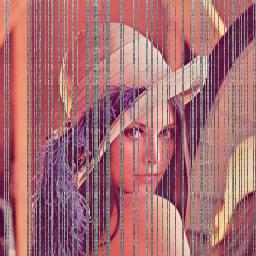}\vspace{-0.3cm}\caption*{\small{SiLRTC-TT}}
\end{minipage}\vspace{-0.2cm}
\begin{minipage}[t]{0.1\textwidth}
\includegraphics[width=1.0\linewidth, height=1.0\linewidth]{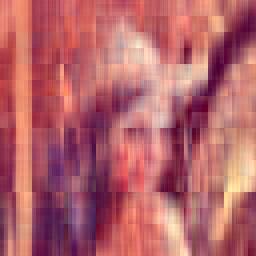}\vspace{-0.3cm}\caption*{TR-ALS}
\end{minipage}
\begin{minipage}[t]{0.1\textwidth}
\includegraphics[width=1.0\linewidth, height=1.0\linewidth]{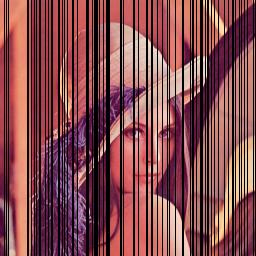}\vspace{-0.3cm}\caption*{tSVD}
\end{minipage}
\begin{minipage}[t]{0.1\textwidth}
\includegraphics[width=1.0\linewidth, height=1.0\linewidth]{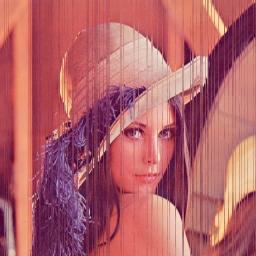}\vspace{-0.3cm}\caption*{\bf TRNNM}
\end{minipage}
\vspace{-0.1cm}
\caption{Lena image with 30\% of stripes missing and its restorations by completion methods.}\label{recovery:lena}\vspace{-0.2cm}
\end{figure}

\vspace{-0.2cm}\begin{figure*}[htb]
\begin{center}
\begin{minipage}[t]{0.08\textwidth}\centering
\includegraphics[width=1.0\linewidth, height=1.0\linewidth]{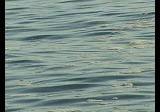}
\end{minipage}\vspace{-0.1cm}
\begin{minipage}[t]{0.08\textwidth}\centering
\includegraphics[width=1.0\linewidth, height=1.0\linewidth]{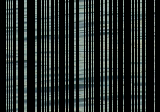}
\end{minipage}
\begin{minipage}[t]{0.08\textwidth}\centering
\includegraphics[width=1.0\linewidth, height=1.0\linewidth]{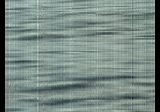}
\end{minipage}
\begin{minipage}[t]{0.08\textwidth}\centering
\includegraphics[width=1.0\linewidth, height=1.0\linewidth]{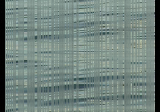}
\end{minipage}
\begin{minipage}[t]{0.08\textwidth}\centering
\includegraphics[width=1.0\linewidth, height=1.0\linewidth]{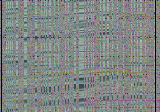}
\end{minipage}
\begin{minipage}[t]{0.08\textwidth}\centering
\includegraphics[width=1.0\linewidth, height=1.0\linewidth]{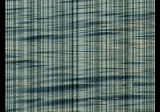}
\end{minipage}
\begin{minipage}[t]{0.08\textwidth}\centering
\includegraphics[width=1.0\linewidth, height=1.0\linewidth]{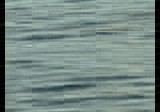}
\end{minipage}
\begin{minipage}[t]{0.08\textwidth}\centering
\includegraphics[width=1.0\linewidth, height=1.0\linewidth]{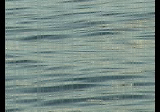}\vspace{2mm}
\end{minipage}\\
\begin{minipage}[t]{0.08\textwidth}\centering
\includegraphics[width=1.0\linewidth, height=1.0\linewidth]{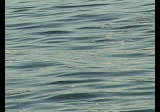}\vspace{-0.2cm}
\caption*{Original}
\end{minipage}
\begin{minipage}[t]{0.08\textwidth}\centering
\includegraphics[width=1.0\linewidth, height=1.0\linewidth]{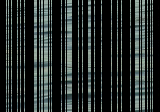}\vspace{-0.2cm}
\caption*{$mr=0.7$}
\end{minipage}
\begin{minipage}[t]{0.08\textwidth}\centering
\includegraphics[width=1.0\linewidth, height=1.0\linewidth]{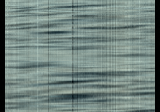}\vspace{-0.2cm}
\caption*{FBCP}
\end{minipage}
\begin{minipage}[t]{0.08\textwidth}\centering
\includegraphics[width=1.0\linewidth, height=1.0\linewidth]{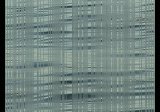}\vspace{-0.2cm}
\caption*{HaLRTC}
\end{minipage}
\begin{minipage}[t]{0.08\textwidth}\centering
\includegraphics[width=1.0\linewidth, height=1.0\linewidth]{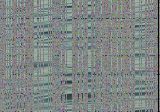}\vspace{-0.2cm}
\caption*{\footnotesize{SiLRTC-TT}}
\end{minipage}
\begin{minipage}[t]{0.08\textwidth}\centering
\includegraphics[width=1.0\linewidth, height=1.0\linewidth]{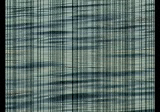}\vspace{-0.2cm}
\caption*{tSVD}
\end{minipage}
\begin{minipage}[t]{0.08\textwidth}\centering
\includegraphics[width=1.0\linewidth, height=1.0\linewidth]{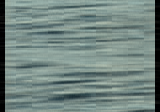}\vspace{-0.2cm}
\caption*{TR-ALS}
\end{minipage}
\begin{minipage}[t]{0.08\textwidth}\centering
\includegraphics[width=1.0\linewidth, height=1.0\linewidth]{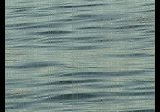}\vspace{-0.2cm}
\caption*{\bf TRNNM}
\end{minipage}\vspace{-0.7cm}
\end{center}
\caption{The 10th and 20th frames (from top to bottom) of the ocean video with 70\% of stripes missing and its restorations by tensor completion methods.}\label{recovery:ocean}\vspace{-0.6cm}
\end{figure*}

\vspace{-0.4cm}\begin{figure}[htb]
\centering
\subfloat[]{\includegraphics[width=0.23\textwidth]
{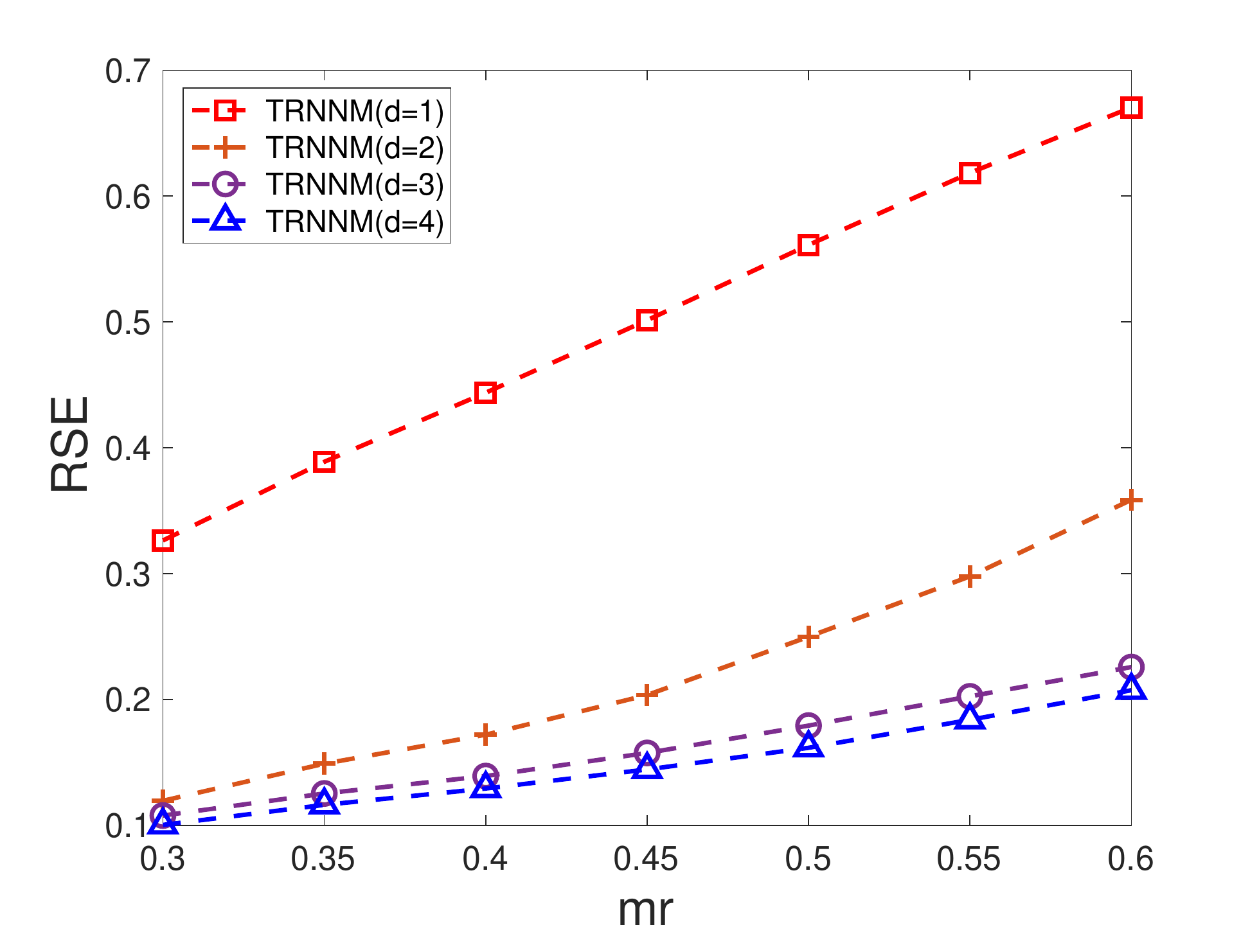}}
\subfloat[]{\includegraphics[width=0.23\textwidth]
{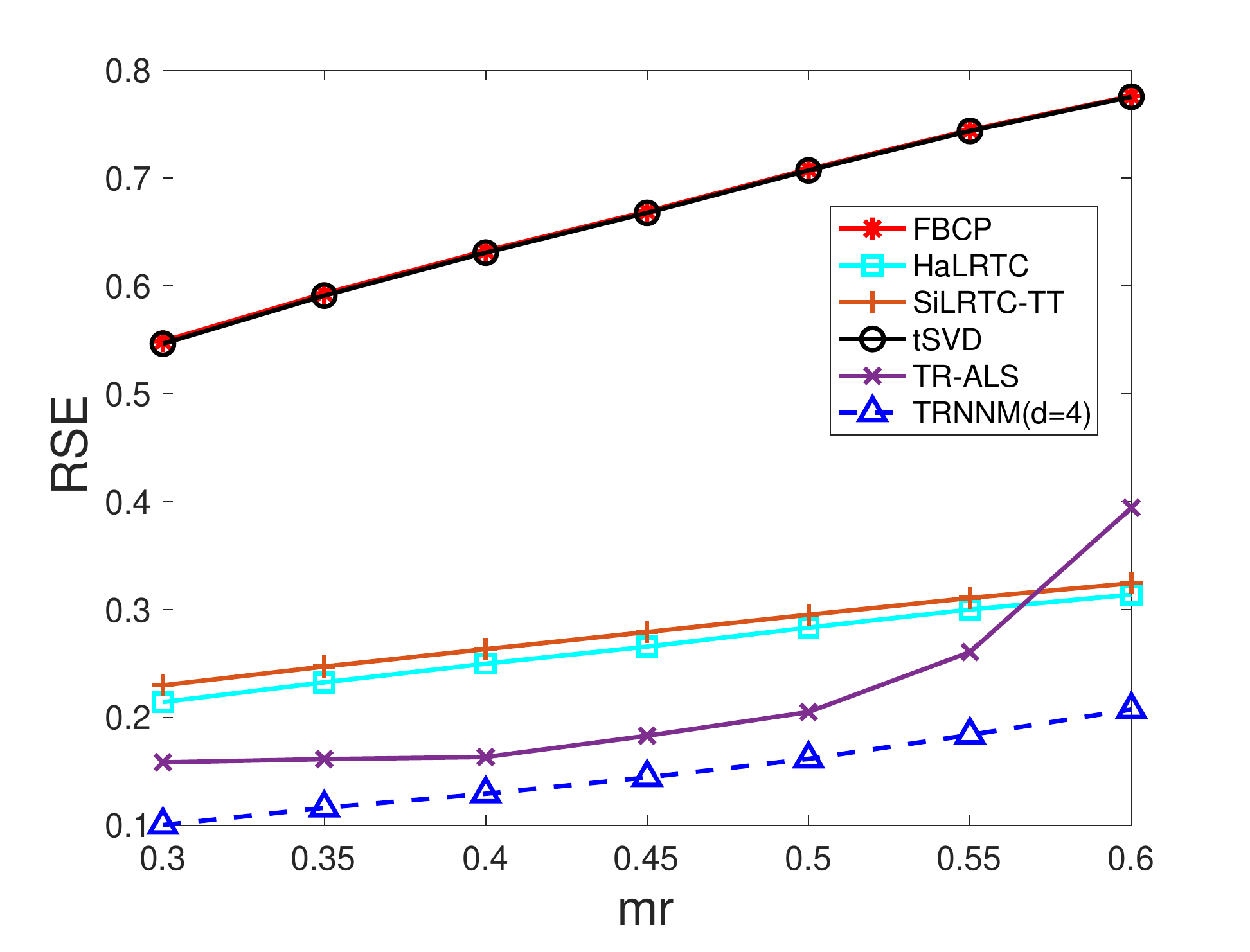}}\vspace{-0.3cm}
\caption{RSE under varying ratio of stripes missing from Lena image. (a) is the investigation of TRNNM with various $d$, and (b) is the comparison of TRNNM with other methods. }
\label{plot:Perf_lena}\vspace{-0.3cm}
\end{figure}

\begin{table}
\centering 
\caption{RSE and runtime (seconds) for 80\%, 70\% and 60\% of stripes randomly missing from the ocean video. } \vspace{-0.2cm}\label{tab:Perf_ocean}
\renewcommand{\arraystretch}{1.3}
\doublerulesep=5pt
\scalebox{0.72}{
\begin{tabular}{c|cc|cc|cc}

 ~     & \multicolumn{2}{c|}{$mr = 0.8$}  & \multicolumn{2}{c|}{$mr = 0.7$}  &  \multicolumn{2}{c}{$mr = 0.6$}   \\  
 \hline
   ~             & RSE &  RunTime  &  RSE &  RunTime  &  RSE &  RunTime \\
   \hline
 FBCP   &   0.207  &   298.41 & 0.107  &    302.40  &  0.099  &  227.91      \\
   
 HaLRTC    & 0.207  &  41.65    &  0.125  &   41.43    & 0.093  &    32.77     \\
    
 SiLRTC-TT   &  0.303 &   81.47   &  0.287  &   43.48    & 0.268 &   40.63    \\

 tSVD     & 0.461    &    506.71    &  0.286  &  498.24     & 0.126 &  453.92   \\

 TR-ALS  & 0.116   &   568.42     & 0.113   &  851.37     &  0.103  &   1.54e3   \\

 {\bf TRNNM}  & {\bf 0.098}   &  235.52    & {\bf 0.066}    &  233.71     & {\bf 0.048}  &   215.01    \\
\end{tabular}
}\vspace{-0.6cm}
\end{table}
\section{Experiments}\label{section5}
To validate our proposed method TRNNM, both image and video completion are used to compare our method and state-of-the-art methods, i.e., FBCP \cite{zhao2015bayesian}, HaLRTC \cite{liu2013tensor}, SiLRTC-TT \cite{bengua2017efficient}, tSVD \cite{zhang2015exact}, TR-ALS \cite{wang2017efficient}. We conduct each experiment  10 times and  record the average relative square error (RSE) and its runtime, where
$
RSE = \|\mathcal{X}-\mathcal{T}\|_{F}/ \|\mathcal{T}\|_{F}.
$
\subsection{Image Completion}
In this section, Lena image is used to evaluate the performance of TRNNM and its compared methods. The image is initially presented by 3rd-order tensor with size of $256\times256\times 3$. We directly reshape the image into 9th-order tensor with size of $4\times\dots\times4\times3$ for the methods assuming the data with TT/TR structure, i.e. SiLRTC-TT, TR-ALS and TRNNM, due to the high-order reshape operation is often used to improve the performance in classification \cite{novikov2015tensorizing} and completions \cite{bengua2017efficient, Yuan2017highorder, wang2017efficient}. We conduct experiments under varying ratio of stripes missing from the image. In the case of $N$th-order tensor, the TRNNM with $d=s$ is equivalent to that with $d=N-s$, we thus only consider the {\it step-length} $d$ varying from 1 to $\left\lfloor\frac{N}{2}\right\rfloor$ where $\left\lfloor.\right\rfloor$ denotes the $floor$ operation. 

Fig. \ref{plot:Perf_lena}(a) shows that our proposed method obtains better results as {\it step-length} increases from 1 to 4, which indicates that closer to $\left\lfloor\frac{N}{2}\right\rfloor$ the {\it step-length} $d$,  stronger the ability of TRNNM to capture information. Seen from Fig. \ref{plot:Perf_lena}(b), TRNNM with $d=4$ significantly outperforms other methods in RSE under various missing ratios. The image restorations are shown in Fig. \ref{recovery:lena}, which demonstrates that TRNNM with $d=4$ performs best on estimating the stripes missing values.

Due to the superiority of TRNNM when {\it step-length} $d=\left\lfloor\frac{N}{2}\right\rfloor$, we set $d=\left\lfloor\frac{N}{2}\right\rfloor$ in default in our later experiments.

\subsection{Video Completion}
The ocean video \cite{liu2009tensor} with size $112\times 160\times 3\times 32$ is used in this experiment. As done before, the ocean video is reshaped into 7th-order tensor of size $16\times7\times16\times10\times3\times8\times4$ for TT/TR-based methods. We conduct experiment under varying ratio of stripes randomly missing from the video.

Table \ref{tab:Perf_ocean} shows that TRNNM significantly performs better than other methods in RSE under our considered missing ratios, i.e. $mr=0.6, 0.7, 0.8$, with an acceptable time-cost. TR-ALS obtains rather well performance following TRNNM, however with significantly time-cost, which is not applicable in practice. For the case of $mr=0.7$, the restorations of some frames are shown in Fig. \ref{recovery:ocean}. Observe that TRNNM obtains the detail information of the frames with a better resolution, which demonstrates the superiority of TRNNM on capturing the information of the video with stripes missing.

%% file: Conclusion.tex
\section{Conclusions}\label{section6}
We propose a convex completion method by minimizing tensor ring nuclear norm which is first introduced in our paper. The proposed method not only has lower computational complexity than the previous TR-based methods, but also avoids choosing the optimal TR rank. Extensive experimental results demonstrate that the proposed method outperforms the conventional tensor completion methods in image/video completion problem.